\documentclass[letterpaper, conference]{ieeeconf}

\IEEEoverridecommandlockouts                              
                                                          
\usepackage{amsmath,amssymb,amsfonts,amsthm}
\usepackage{algorithm}
\usepackage[noend]{algpseudocode}
\usepackage{xspace}
\usepackage{bm}
\usepackage{xcolor}

\usepackage{graphicx}
\usepackage{xspace}
\usepackage{comment}
\algrenewcommand\algorithmicindent{1.0em}%

\allowdisplaybreaks
\newcommand{\ljmargin}[2]{{\color{cyan}#1}\marginpar{\color{cyan}\raggedright\tiny[LJ]:#2}}

\makeatletter
\let\NAT@parse\undefined
\makeatother
\usepackage[colorlinks=true, bookmarks=true, pagebackref=true,citecolor=blue]{hyperref}

\title{\Large \bf Revisiting the Asymptotic Optimality of RRT*}

\author{Kiril Solovey$^1$, Lucas Janson$^2$, Edward Schmerling$^3$, Emilio Frazzoli$^4$, and Marco Pavone$^1$
\thanks{$^{1}$ Department of Aeronautics and Astronautics, Stanford University, CA, USA.}%
\thanks{$^{2}$ 
Department of Statistics, Harvard University, MA, USA.}%
\thanks{$^{3}$ 
Waymo Research, CA, USA.}%
\thanks{$^{4}$ 
Institute for Dynamic Systems and Control, ETH Zurich, Switzerland.}%
}

\begin{document}

\maketitle
\thispagestyle{empty}
\pagestyle{empty}

\newcommand{\ignore}[1]{}

\def\vor{\text{Vor}}

\def\P{\mathcal{P}} \def\C{\mathcal{C}} \def\H{\mathcal{H}}
\def\F{\mathcal{F}} \def\U{\mathcal{U}} \def\L{\mathcal{L}}
\def\O{\mathcal{O}} \def\I{\mathcal{I}} \def\S{\mathcal{S}}
\def\G{\mathcal{G}} \def\Q{\mathcal{Q}} \def\I{\mathcal{I}}
\def\T{\mathcal{T}} \def\L{\mathcal{L}} \def\N{\mathcal{N}}
\def\V{\mathcal{V}} \def\B{\mathcal{B}} \def\D{\mathcal{D}}
\def\W{\mathcal{W}} \def\R{\mathcal{R}} \def\M{\mathcal{M}}
\def\X{\mathcal{X}} \def\A{\mathcal{A}} \def\Y{\mathcal{Y}}
\def\L{\mathcal{L}}

\def\dS{\mathbb{S}} \def\dT{\mathbb{T}} \def\dC{\mathbb{C}}
\def\dG{\mathbb{G}} \def\dD{\mathbb{D}} \def\dV{\mathbb{V}}
\def\dH{\mathbb{H}} \def\dN{\mathbb{N}} \def\dE{\mathbb{E}}
\def\dR{\mathbb{R}} \def\dM{\mathbb{M}} \def\dm{\mathbb{m}}
\def\dB{\mathbb{B}} \def\dI{\mathbb{I}} \def\dM{\mathbb{M}}
\def\dZ{\mathbb{Z}}
\def\E{\mathbf{E}} 

\def\eps{\varepsilon}

\def\limn{\lim_{n\rightarrow \infty}}

\newcommand{\ch}{\mathrm{ch}}
\newcommand{\pspace}{{\sc pspace}\xspace}
\newcommand{\threesum}{{\sc 3Sum}\xspace}
\newcommand{\np}{{\sc np}\xspace}
\newcommand{\degree}{\ensuremath{^\circ}}
\newcommand{\argmin}{\operatornamewithlimits{argmin}}

\newcommand{\dist}{\textup{dist}}

\newtheorem{lemma}{Lemma}
\newtheorem{theorem}{Theorem}
\newtheorem{corollary}{Corollary}
\newtheorem{claim}{Claim}
\newtheorem{proposition}{Proposition}

\theoremstyle{definition}
\newtheorem{definition}{Definition}
\newtheorem{remark}{Remark}
\theoremstyle{plain}
\newtheorem{observation}{Observation}

\def\aas{a.a.s.\xspace}
\def\0{\bm{0}}

\def\EE{\mathfrak{E}}

\def\todo#1{\textcolor{blue}{\textbf{TODO:} #1}}
\def\new#1{\textcolor{magenta}{#1}}
\def\kiril#1{\textcolor{magenta}{\textbf{Kiril:}}}
\def\removed#1{\textcolor{green}{#1}}

\def\dx{\,\mathrm{d}x}
\def\dy{\,\mathrm{d}y}
\def\drho{\,\mathrm{d}\rho}

\newcommand{\prm}{{\tt PRM}\xspace}
\newcommand{\prmstar}{{\tt PRM}$^*$\xspace}
\newcommand{\rrt}{{\tt RRT}\xspace}
\newcommand{\est}{{\tt EST}\xspace}
\newcommand{\grrt}{{\tt GEOM-RRT}\xspace}
\newcommand{\rrtstar}{{\tt RRT}$^*$\xspace}
\newcommand{\rrg}{{\tt RRG}\xspace}
\newcommand{\btt}{{\tt BTT}\xspace}
\newcommand{\fmt}{{\tt FMT}$^*$\xspace}
\newcommand{\dfmt}{{\tt DFMT}$^*$\xspace}
\newcommand{\dprm}{{\tt DPRM}$^*$\xspace}
\newcommand{\mstar}{{\tt M}$^*$\xspace}
\newcommand{\drrtstar}{{\tt dRRT}$^*$\xspace}
\newcommand{\sst}{{\tt SST}\xspace}
\newcommand{\aorrt}{{\tt AO-RRT}\xspace}

\newcommand{\xmin}{x_{\text{min}}}
\newcommand{\Xnear}{X_{\text{near}}}
\newcommand{\xinit}{x_{\text{init}}}
\newcommand{\xgoal}{x_{\text{goal}}}
\newcommand{\xnew}{x_{\text{new}}}
\newcommand{\xnear}{x_{\text{near}}}
\newcommand{\xrand}{x_{\text{rand}}}
\newcommand{\xparent}{x_{\text{parent}}}
\newcommand{\cmin}{c_{\text{min}}}

\newcommand{\samplefree}{\textsc{sample-free}}
\newcommand{\nearest}{\textsc{nearest}}
\newcommand{\near}{\textsc{near}}
\newcommand{\steer}{\textsc{steer}}
\newcommand{\collision}{\textsc{collision-free}}
\newcommand{\cost}{\textsc{cost}}
\newcommand{\parent}{\textsc{parent}}

\begin{abstract}
RRT$\bm{^*}$ is one of the most widely used sampling-based algorithms for asymptotically-optimal motion planning. RRT$\bm{^*}$  laid the foundations for optimality in motion planning as a whole, and inspired the development of numerous new algorithms in the field, many of which build upon RRT$\bm{^*}$ itself. In this paper, we first identify a logical gap in the optimality proof of RRT$\bm{^*}$, which was developed by Karaman and Frazzoli (2011).  Then, we present an alternative and mathematically-rigorous proof for asymptotic optimality. Our proof suggests that the connection radius used by RRT$\bm{^*}$ should be increased from $\bm{\gamma \left(\frac{\log n}{n}\right)^{1/d}}$ to  $\bm{\gamma' \left(\frac{\log n}{n}\right)^{1/(d+1)}}$ in order to account for the additional dimension of time that dictates the samples' ordering. Here $\bm{\gamma, \gamma'}$ are constants, and $\bm{n, d}$ are the number of samples and the dimension of the problem, respectively.
\end{abstract}

\section{Introduction}\label{sec:introduction}
For many robot motion-planning applications, feasibility is not enough---we further desire path plans that are of high quality, 
reflecting a need for robots that can achieve their goals with efficiency, alacrity, and economy of motion. To this end we seek planning algorithms that can be trusted, whatever obstacle environment a robot faces, to produce optimal or near-optimal plans with minimal scenario-specific tuning. The advent of the asymptotically-optimal rapidly-exploring random tree (\rrtstar) algorithm~\cite{KF11} has ushered in a decade of theoretical and practical successes in the development of optimal sampling-based motion-planning algorithms. 

Although proposed in its initial form for the case of minimum-length path planning for robots without dynamic constraints, \rrtstar has been extended to handle kinodynamic planning problems~\cite{KaramanFrazzoli2010} including robotic systems governed by non-holonomic constraints~\cite{KaramanFrazzoli2013}, more expressive costs accounting for robot energy expenditure~\cite{GoretkinPerezEtAl2013,WebbBerg2013}, and even to plan paths that minimize violation of safety rules~\cite{ReyesCastroChaudhariEtAl2013} or that otherwise balance performance considerations with safety constraints~\cite{LiuAng2014}. Heuristic modifications to the core algorithm have also been demonstrated that improve practical \rrtstar implementations~\cite{AkgunStilman2011,GammellSrinivasaEtAl2014}. 

Each of these extensions leverages the simple yet powerful iterative local graph-rewiring technique introduced by \rrtstar to enable convergence to the optimal solution (as computation budget increases), provided an appropriate choice for the scaling of the rewiring radius as a function of sample count. Moreover, each of these extensions draws upon the original analysis presented in~\cite{KF11} for the fundamental asymptotic scaling of this algorithm parameter; this analysis is therefore core to each of their optimality guarantees.\vspace{5pt}

\noindent\emph{Contribution.} The primary contribution of this paper is an in-depth study of the theoretical analysis underpinning the asymptotic-optimality criterion for the \rrtstar algorithm. In revisiting this analysis, we identify a logical gap in the original proof and provide an amended proof suggesting a larger
radius scaling exponent to ensure asymptotic optimality. The impact of this paper is potentially far-reaching in the large number of works that currently appeal to \rrtstar optimality to make their theoretical guarantees.

The paper is organized as follows. Section~\ref{sec:preliminaries} provides preliminaries and a description of \rrtstar. In Section~\ref{sec:wrong} we review the original optimality proof of \rrtstar and identify a logical gap within it. In Section~\ref{sec:main} we provide the main contribution of this paper, which is an alternative proof that circumvents this logical gap. We conclude the paper in Section~\ref{sec:conclusion}.

\section{Preliminaries}\label{sec:preliminaries}
We provide several basic definitions that will be used throughout the
paper. Given two points $x,y\in \dR^d$, denote by $\|x-y\|$ the
standard Euclidean distance. Denote by $\B_{r}(x)$ the $d$-dimensional
ball of radius $r>0$ centered at $x\in \dR^d$. Define
$\B_{r}(\Gamma) := \bigcup_{x \in \Gamma}\B_{r}(x)$ for any
$\Gamma \subseteq \dR^d$.  Similarly, given a curve
$\sigma:[0,1]\rightarrow \dR^d$, define
$\B_r(\sigma)=\bigcup_{\tau\in[0,1]}\B_r(\sigma(\tau))$. For a subset
$D\subset \dR^d$, $|D|$ denotes its Lebesgue measure. All
logarithms used herein are to base~$e$. 

\subsection{Motion planning}
Denote by $\C$ the robot's configuration space, and by
$\F\subseteq \C$ the free space, i.e., the set of all collision free
configurations. 
We assume that $\C$ is a subset of the Euclidean space. For simplicity,
let $\C=[0,1]^d\subset\dR^d$ for some fixed $d\geq 2$. 
Given start and target configurations $s,t\in \F$, the \emph{motion-planning} problem
consists of finding a continuous path (curve)
$\sigma:[0,1]\rightarrow \F$ such that $\sigma(0)=s$ and $\sigma(1)=t$. That is, the
robot starts its motion along $\sigma$ at $s$, and ends at $t$, while
avoiding collisions. An instance of the problem is defined by $(\F,s,t)$.  We consider the standard path length as a measure of quality:

\begin{definition}\label{def:length}
  Given a path $\sigma$, its \emph{length} (cost), which corresponds to its 
  Hausdorff measure, is represented by
  \[c(\sigma)=\sup_{n\in \dN_+, 0=\tau_1\leq \ldots \leq \tau_n
    =1}\sum_{i=2}^n\|\sigma(\tau_i)-\sigma(\tau_{i-1})\|.\]
\end{definition}
	
We proceed to describe the notion of \emph{robustness}, which is
essential when discussing theoretical properties of sampling-based planners.
Given a subset $\Gamma\subset \C$ and two configurations
$x,y\in \Gamma$, denote by $\Sigma_{x,y}^\Gamma$ the set of all
continuous paths, whose image is in $\Gamma$, that start in $x$ and
end in $y$, i.e., if $\sigma\in \Sigma_{x,y}^\Gamma$ then
$\sigma:[0,1]\rightarrow \Gamma$ and $\sigma(0)=x,\sigma(1)=y$. We mention that the following definition is slightly different than the one used in~\cite{KF11,JSCP15}.
	
\begin{definition}\label{def:robustly_feasible}
  Let $(\F,s,t)$ be a motion-planning problem. A path
  $\sigma \in \Sigma_{s,t}^\F$ is \emph{robust} if there exists $\delta >0$
  such that $\B_{\delta}(\sigma)\subset \F$.  We also say that $(\F,s,t)$
  is \emph{robustly feasible} if there exists such a robust path.
\end{definition}
	
\begin{definition}\label{def:robust_opt_len}
  The \emph{robust optimum} is defined as
  \[c^*=\inf\left\{c(\sigma)\middle|\sigma \in \Sigma_{s,t}^\F\textup{ is
      robust}\right\}.\]
	\end{definition}
	
\subsection{Algorithms}
While our main focus in this paper is the \rrtstar algorithm, we also rely on the properties of the \rrt algorithm, which is
described first. The following description of the (geometric) \rrt
algorithm is based on~\cite{KL00} and~\cite{KF11}.

\begin{algorithm}
	\caption{\tt RRT($\xinit:=s, \xgoal:=t, n, \eta$)}
	\label{alg:rrt}
	\small
	\begin{algorithmic}[1]
		\State{$V=\{\xinit\}$}
		\For {$j = 1 \text{ to } n$}
		\State $\xrand\gets \samplefree(\;)$
		\State $\xnear\gets \nearest(\xrand, V)$
		\State  $\xnew \gets$ \steer($\xnear, \xrand, \eta$)
		\If {\collision($\xnear, \xnew$)} 
		\State{$V=V\cup \{\xnew\}$; $E=E\cup \{(\xnear, \xnew)\}$}
		\EndIf
		\EndFor	
		\State 		\Return $G=(V,E)$
	\end{algorithmic}
\end{algorithm}

The input for \rrt (Algorithm~\ref{alg:rrt}) is an initial and goal  configurations $\xinit, \xgoal$, number of iterations $n$, and a steering parameter $\eta>0$.  \rrt constructs a tree $G=(V,E)$ by performing $n$ iterations. In each iteration, a new sample $\xrand$ is returned from $\F$ uniformly at random by calling \samplefree. Then, the vertex $\xnear\in V$ that is nearest (according to $\|\cdot\|$) to $\xrand$ is found using \nearest. A new configuration $\xnew\in\X$ is then returned by \steer, such that $\xnew$ is on the line segment between $\xnear$ and $\xrand$, and the distance $\|\xnear - \xnew\|$ is at most $\eta$. Finally, \collision($\xnear, \xnew$) checks whether the straight-line path from $\xnear$ to $\xnew$ is collision free. If so, $\xnew$ is added as a vertex to $G$ and is connected by an edge from $\xnear$. 


We proceed to describe \rrtstar~\cite{KF11} in Algorithm~\ref{alg:rrt_star}. Every \rrtstar iteration begins with an \rrt-style extension. The difference lies in the subsequent lines. First, \rrtstar attempts to connect the tree to  $\xnew$ from all its neighbors in $V$ within a $\min\{r(|V|),\eta\}$ vicinity (lines~7-15). Notice that the expression $r(|V|)$ determines the radius based on the current number of vertices in $V$. 
(The operation $\near(\xnew, V, \min\{r(|V|), \eta\})$ returns the subset $V\cap \B_{\min\{r(|V|), \eta\}}(\xnew)$, i.e., the vertices that are within a distance of $\min\{r(|V|), \eta\}$ from $\xnew$.) However, it only adds a single edge to $\xnew$ from the neighbor $\xmin\in \Xnear$ such that $\cost(\xnew)$ is minimized (line~16). In the next step, \rrtstar attempts to perform rewires (lines 17-21): with the addition of $\xnew$, it may be beneficial to reroute the existing path of $\xnear$ to use $\xnew$. \rrtstar checks whether changing the parent of $\xnear$ to be $\xnew$ reduces $\cost(\xnear)$. ($\parent(\xnear)$ returns the immediate predecessor of $\xnear$ in $G$. $\cost(x)$ for $x\in V$ returns the cost of the path leading from $\xinit$ to $x$ in $G$.)

\begin{algorithm}
	\caption{\tt RRT$^*$($\xinit:=s, \xgoal:=t, n, r, \eta$)}
	\small
	\label{alg:rrt_star}
	\begin{algorithmic}[1]
		\State{$V=\{\xinit\}$}
		\For {$j = 1 \text{ to } n$}
		\State $\xrand\gets \samplefree(\;)$
		\State $\xnear\gets \nearest(\xrand, V)$
		\State  $\xnew \gets$ \steer($\xnear, \xrand, \eta$)
		\If {\collision($\xnear, \xnew$)} 
        \State $\Xnear=\near(\xnew, V, \min\{r(|V|), \eta\})$
        \State $V=V\cup \{\xnew\}$
        \State $\xmin=\xnear$
        \State $\cmin=\cost(\xnear)+\|\xnew - \xnear\|$
        \For {$\xnear \in \Xnear$}
        \If {$\collision(\xnear, \xnew)$}
        \If {$\cost(\xnear)+\|\xnew - \xnear\|<\cmin$}
        \State $\xmin=\xnear$ 
        \State $\cmin=\cost(\xnear)+\|\xnew - \xnear\|$
        \EndIf
        \EndIf
        \EndFor
        \State $E=E\cup \{(\xmin, \xnew)\}$
        \For {$\xnear \in \Xnear$}
        \If {$\collision(\xnew, \xnear)$}
        \If {$\cost(\xnew)+\|\xnear - \xnew\|<\cost(\xnear)$}
        \State $\xparent =\parent(\xnear) $
        \State $E=E\cup \{(\xnew, \xnear)\}\setminus\{(\xparent,\xnear)\}$
        \EndIf
        \EndIf
        \EndFor
		\EndIf
		\EndFor	
		\State 		\Return $G=(V,E)$
	\end{algorithmic}
\end{algorithm}

\begin{remark}\label{remark:steering}
  As mentioned above, \rrtstar performs extensions of the tree in a manner similar to \rrt. That is, \steer\ generates $\xnew$, which lies on the straight line connecting $\xnear,\xrand$, such that $\|\xnew-\xnear\|\leq \eta$. Note that initially $\xnew \neq \xrand$, but once the space is sufficiently covered by $G$, i.e., when $\F\subset \bigcup_{v\in V}\B_\eta(v)$,
then in all the following iterations it will hold that 
  $\xnew=\xrand$. This property will be important in the analysis of
  \rrtstar, as it indicates that  $\xnew$ is uniformly sampled from $\F$. This notion will be formalized below. For now, it is useful to note that given the same sequence of samples, \rrt and \rrtstar will generate two (possibly distinct) graphs that have a common vertex set. 
\end{remark}

\section{Original optimality proof}\label{sec:wrong}
In this section we review the original proof~\cite{KF11} for asymptotic optimality of \rrtstar, and point out a logical gap. Specifically, Theorem~38 in~\cite{KF11}  states that if the connection radius used by \rrtstar is of the form
\begin{align}
r^{\text{KF}}(n)=\gamma^{\text{KF}}\left(\frac{\log n}{n}\right)^{1/d},
\label{eq:wrong-r(n)}
\end{align}
where $n\in \dN_+$, and for some constant $\gamma^{\text{KF}}>0$, the cost of the solution obtained by \rrtstar converges to the robust optimum~$c^*$ as $n\to \infty$, almost surely. 

\begin{figure*}[!th]
    \vspace{5pt}
  \begin{center}
    \includegraphics[width=1.4\columnwidth, page=5]{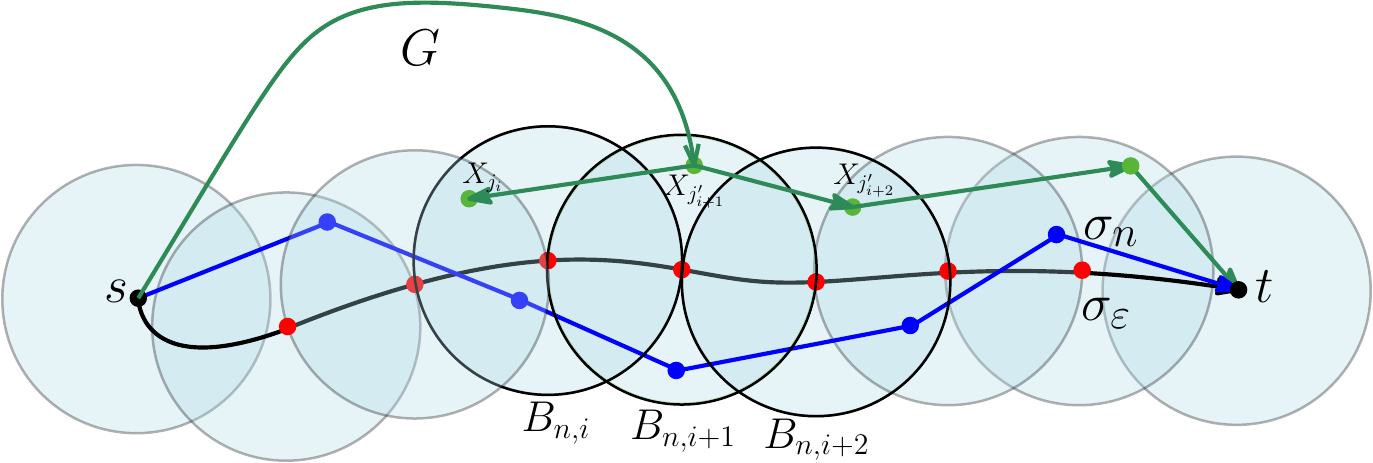}
  \end{center}
  \caption{Illustration of the components in the original proof~\cite{KF11}. (a) The robustly-optimal path $\sigma_\eps$ is drawn as a black curve. (b) Discs represent the balls $B_{n,1},\ldots, B_{n,M_n}$, whose centers are denoted as red bullets along $\sigma_\eps$. The path $\sigma_n$ connecting samples between adjacent balls in an increasing order is illustrated as a blue curve. (c) A problematic scenario (Section~\ref{sec:issue}), where the \rrtstar tree $G$ yields a suboptimal solution, is depicted in green.}
  \label{fig:original}
  \vspace{-10pt}
\end{figure*}

\subsection{Review of previous proof}
We provide a sketch of the original proof and identify a logical gap. We mention that our definitions of robustness (Definition~\ref{def:robustly_feasible}) and robust optimum (Definition~\ref{def:robust_opt_len}) are  simplified versions of the ones used originally in~\cite{KF11}, where the latter are slightly less convenient to work with (especially in correction of the proof which we give in Section~\ref{sec:main}).  We thus adapt the original proof details presented in this section to our setting. We emphasize that the logical gap is unrelated to those definitions, and our argument presented below can be easily remapped to the original formulation.

Recall that the sample set of \rrtstar consists of $n$ time-labeled configurations. Denote by $\{X_1,\ldots,X_n\}$ the sample set, where indices denote the order in which the samples are drawn.  Fix $\eps>0$ and let $\sigma_\eps$ be a robust solution path such that $c(\sigma_\eps)\leq (1+\eps)c^*$. The proof constructs a sequence of $M_n\leq n$ identical balls $B_{n,1},\ldots, B_{n,M_n}$ that are centered on some equally-spaced points along $\sigma_\eps$. The size and spacing of balls is set so that (a) $\sigma_\eps$ is completely covered by them, (b) $\bigcup_{i=1}^{M_n} B_{n,i}\subseteq \F$, and (c) for every $1\leq i \leq M_n, x\in B_{n,i}, x'\in B_{n,i+1}$ it holds that $\|x-x'\|\leq r^{\text{KF}}(n)\leq r^{\text{KF}}(|V|)$. Furthermore, it is shown in~\cite{KF11} that given $x_i\in B_{n,i}$ for every $1\leq i\leq M_n$, the length of the path $\sigma$ connecting each $x_i$ to the point in the next ball with a straight line converges (as $n\rightarrow \infty$) to the length of $\sigma_\eps$ (see Figure~\ref{fig:original}).

The proof establishes that if for every $1\leq i < M_n$ there exist $X_{j_i},X_{j_{i+1}}$ such that (i) $X_{j_i}\in B_{n,i},X_{j_{i+1}}\in B_{n,i+1}$ and (ii) $j_i<j_{i+1}$, then \rrtstar is asymptotically optimal (see Section~G.3 in~\cite{KF11}). The rationale behind these conditions is as follows. 
Condition (i) makes sure that the optimal path is approximated by samples drawn by \rrtstar, i.e., for every point along $\sigma_\eps$ there is a sample point in its vicinity. Condition (ii)  ensures that \rrtstar will have the opportunity to add a directed edge from $X_{j_i}$ to $X_{j_{i+1}}$: as $X_{j_{i+1}}$ is sampled after $X_{j_i}$ then \rrtstar would consider drawing a directed edge from the latter to the former, considering the fact that  $X_{j_{i}}\in\near(X_{j_{i+1}}, V, \min\{r(n), \eta\})$ (this is formalized in Claim~\ref{clm:connection} below). Observe that $r(n)$ is used as a conservative lower-bound for $r(|V|)$ throughout~\cite{KF11}, as we do too.

Consequently, the proof deduces that if these conditions are met \rrtstar is guaranteed to find  a  solution with cost at most $c^*(1+\eps)$ with probability that converges to $1$ as $n\rightarrow\infty$. In particular, denote by  $X_{j_1},\ldots,X_{j_{M_n}}$ the sequence of samples satisfying the conditions above, and let $\sigma_n$ be a path that is induced by those $M_n$ samples in the prescribed order. Then the claim is that the solution returned by \rrtstar is of length $c(\sigma_n)$, if not shorter. 

\subsection{A logical gap}\label{sec:issue}
We identify an issue with the proof technique described above, and in particular with the conditions (i) and (ii). \textbf{We assert that the line of reasoning mentioned above overlooks the fact that the existence of pairwise sequential samples does not directly imply the existence of a whole chain of samples with a proper ordering such that a path in $\bm{G}$ traces through all the balls in sequence.} That is, the fact that for every $1\leq i < M_n$ (i) there exist $X_{j_i},X_{j_{i+1}}$ such that $X_{j_i}\in B_{n,i},X_{j_{i+1}}\in B_{n,i+1}$ and (ii)  $j_i<j_{i+1}$, does not necessarily mean that (iii) there exists a sequence $j_1\leq j_2\leq \ldots \leq j_{M_n}$ such that $X_{j_i}\in B_{n,i}$ for every $1\leq i < M_n$; (iii) is a sufficient (but not necessary) condition for recovering a path that is at least as good as $\sigma_n$. 

Consider for instance the case where $X_{j_i}\in B_{n,i}, X_{j_{i+1}}\in B_{n,i+1}, X_{j'_{i+1}}\in B_{n,i+1}, X_{j'_{i+2}}\in B_{n,i+2}$ and $j_i<j_{i+1}, j'_{i+1}<j'_{i+2}$, but $j'_{i+2} < j_{i}$, where there are two points $X_{j_{i+1}},X_{j'_{i+1}}$ that fall into the same ball $B_{n,i+1}$ (see Figure~\ref{fig:original} (c)). 
Define $X_1^{j_i-1}=\{X_1,\ldots, X_{j_i - 1}\}, B_1^{i-1}=\bigcup_{k=1}^{i-1}B_{n,k}$, and let $X^B=X_1^{j_i-1}\cap B_1^{i-1}\cap \B_{r(j_i)^{\text{KF}}}(X_{j_i})$. Namely, $X^B$ contains all the sampled points that were drawn before $X_{j_i}$, which lie in previous balls along $\sigma_\eps$, and whose distance from $X_{j_i}$ is at most $r(j_i)^{\text{KF}}$. 

Now assume that $X^B=\emptyset$. We can choose the current structure of $G$ and the locations of $X_{j_i}, X_{j'_{i+1}}$ such that the only directed edge that is added in iteration $j_i$ is $(X_{j'_{i+1}},X_{j_i})$, i.e., from $X_{j'_{i+1}}$ to $X_{j_i}$ (rather than the other way around). Note that in iteration $j_{i+1}$ the addition of sample $X_{j_{i+1}}$ would not resolve this problematic wiring since the latter sample will be connected by a directed edge either from $X_{j_i}$ or $X_{j'_{i+1}}$. Moreover, we can repeat this argument for preceding balls to yield a long chain of samples that are connected in the opposite direction.

In this discussion it is important to keep in mind that \rrtstar performs rewiring (i.e., changing the predecessor of a given vertex) only locally (lines 17-21 of Algorithm~\ref{alg:rrt_star}). That is, in order to force a rewiring of a given vertex $X_j$ \rrtstar must sample a vertex $\xnew$ in the vicinity of $X_j$, and this rewiring would not cause a chain of rewires for $X_j$s predecessors or successors in $G$. Consequently, in order to reverse the direction of the aforementioned chain from $X_{j'_{i+1}}$, \rrtstar would need to sample new vertices along the chain in the correct order. For a more detailed example see the appendix. 

As we show in our proof in the next section, condition (iii) is in fact sufficient to guarantee asymptotic optimality, and we prove that it indeed holds with high probability when we slightly increase the connection radius from Equation~\eqref{eq:wrong-r(n)}, and modify the constant $\gamma^{\text{KF}}$.


\section{Alternative proof}
\label{sec:main}
In order to account for the additional dimension of time, we set the connection radius to be $r(n)=\gamma\left(\frac{\log n}{n}\right)^{\frac{1}{d+1}}$, where $\gamma$ is a constant that will be determined below. We state our main theorem and provide an overview of the proof. The full proof is presented later on. Note that our result suggests that the exponent should be decreased from $1/d$ to $1/(d+1)$, which yields a larger radius overall. 
Denote by $\sigma_n$ the path connecting $s$ to $t$ returned by \rrtstar after $n$ iterations. Recall that $c(\sigma_n)$ denotes its length (in case that no solution is found, the length of $\sigma_n$ is assumed to be $\infty$). 
Our main theorem, which appears below, states that if $\gamma$ is set correctly, then the cost of the solution returned by \rrtstar is upper-bounded asymptotically by $(1+\eps)c^*$, where $c^*$ is the robust optimum, and $\eps$ is a tuning parameter. Additional tuning parameters that appear in the theorem are as follows:  $\eta$ is the steering size of \rrtstar (Algorithm~\ref{alg:rrt_star}, line~5), while $\mu$ and $\theta$ are constants whose purpose will become clear in the proof of the theorem. 

\begin{theorem}\label{thm:main}
  Suppose that $(\F,s,t)$ is robustly feasible and fix $\eta>0$, $\eps\in (0,1)$,\footnote{For simplicity, we upper-bound $\eps$ with $1$ although the proof can be adapted to accommodate larger stretch factors.} $\theta\in (0,1/4)$, and $\mu\in (0,1)$. Define the radius of \rrtstar to be
  \begin{align}
  r(n)& =\gamma \left(\frac{\log n}{n}\right)^{\frac{1}{d+1}},\label{eq:r_n}
  \end{align}
  such that 
  \begin{align}
\gamma & \geq (2+\theta)\left(\frac{(1+\eps/4)c^*}{(d+1)\theta (1-\mu)}\cdot \frac{|\F|}{\zeta_d}\right)^{\frac{1}{d+1}}, \label{eq:gamma}
  \end{align}
   where $\zeta_d$ is the volume of a unit $d$-dimensional hypersphere. Then 
\[\limn \Pr[c(\sigma_n)\leq (1+\eps)c^*]=1.\]
\end{theorem}

Our proof of Theorem~\ref{thm:main} proceeds similarly to the proof of the asymptotic optimality of \fmt~\cite{JSCP15} (which is in turn based on~\cite{KF11}), but with additional complications due to the time dimension and the coupling with the \rrt algorithm. We proceed to describe the main ingredients of the proof. 

Fix the parameters $\eps\in (0,1), \theta\in (0,1/4), \mu \in (0,1), \eta > 0$. Due to the fact that $(\F,s,t)$ is robustly feasible, there exists a robust path $\sigma_{\eps}\in \Sigma_{s,t}^\F$ and $\delta>0$ such that $c(\sigma_{\eps})\leq (1+\eps/4)c^*$ and $\B_{\delta}(\sigma_\eps)\subset \F$. We will show that the \rrtstar graph $G$ contains a path that is in the vicinity of $\sigma_{\eps}$, which implies that the solution returned by \rrtstar is of cost at most $(1+\eps)c^*$ (which is slightly larger than $(1+\eps/4)c^*$ due to the fact that this is still an approximation of the path $\sigma_\eps$).

The first part of the proof deals with the technicality involved with the samples produced by the algorithm. Denote by $V=\{X_1\ldots, X_n\}$ the sequence of vertices produced by \rrtstar, where $X_j$ is equal to $\xnew$ generated in iteration $j$. Due to the fact that \rrtstar (and \rrt) perform steering (line~5), samples are not distributed in a uniform manner, as $\xrand$ is not necessarily identical to $\xnew$ (see Remark~\ref{remark:steering}).  However, we do show that most of the vertices in $V$ that are in the vicinity of $\sigma_\eps$ are distributed uniformly at random, with probability approaching $1$  (see Lemma~\ref{lem:ee1}). This event is denoted by $\EE^1$ (see Definition~\ref{def:ee1}).

Next, we proceed in a manner similar to other proofs of asymptotic optimality (see,~\cite{KF11,JSCP15,SK18}), by defining a sequence of points $x_1,\ldots,x_{M_n}$ along the path $\sigma_\eps$ and specifying a sequence of balls $B_{n,1},\ldots, B_{n,M_n}$ that are centered on those points respectively, and whose radius is proportional to $r(n)$. More formally, define $M_n=\Big\lceil c(\sigma_\eps)\cdot\left(\frac{r(n)}{2+\theta}\right)^{-1}\Big\rceil$,
and let $x_1,\ldots, x_{M_n}$ be a sequence of points along $\sigma_\eps$ such that $\|x_i-x_{i-1}\|\leq \frac{\theta r(n)}{2+\theta}$, $x_1=s,x_{M_n}=t$. For every $1\leq i\leq M_n$ define  $B_{n,i}:=\B_{\frac{r(n)}{2+\theta}}(x_i)$. 

As suggested in Section~\ref{sec:wrong}, we need to reason both about the existence of samples inside those balls, and the order of those samples.  We assign to every ball $B_{n,i}$  a specific time window $T_i$, corresponding to allowed timestamps of samples, and partition the sample set $V=\{X_1,\ldots, X_n\}$ into the subsets $V_0,V_1,\ldots, V_{M_n}$, where $X_j\in V_i$ if $j\in T_i$. In particular, $T_0$ consists of the first $n'$ indices, where $n'=\mu n$, and every $T_i$, where $i>1$ consists of $(n-n')/M_n$ indices, and $\mu\in (0,1)$ is a constant:
\[T_{i\neq 0} = 
    \left\{n'+(i-1)\cdot \left\lfloor\tfrac{n-n'}{M_n}\right\rfloor+1,\ldots, n'+i\cdot \left\lfloor \tfrac{n-n'}{M_n}\right\rfloor \right\}.\] 
We show that the event $\EE^2$ (Definition~\ref{def:ee2}) indicating that every $B_{n,i}$ contains a vertex from $V_i$ occurs with probability approaching $1$ as well (Lemma~\ref{lem:ee2}).
The motivation for this event is the following claim, which indicates that edges between points in consecutive balls are added if deemed beneficial. 
\begin{claim}\label{clm:connection}
    There exists $n\in \dN_+$ large enough such that the following holds with respect to $G_{j_{i+1}}=(V_{j_{i+1}}, E_{j_{i+1}})$:
  Suppose that there exist $X_{j_i}\in  V_i\cap B_{n,i}, X_{j_{i+1}}\in  V_{i+1}\cap B_{n,i+1}$ and denote by $G_{j_{i+1}}$ the \rrtstar graph at the end of iteration $j_{i+1}$. Then in $G_{j_i+1}$ it follows that $\cost(X_{j_{i+1}}) \leq \cost(X_{j_{i}})+\|X_{j_i}-X_{j_{i+1}}\|$.
\end{claim}
\begin{proof}
Recall that $B_{n,i}=\B_{\frac{r(n)}{2+\theta}}(x_i)$ and $\|x_i-x_{i+1}\|\leq \frac{\theta r(n)}{2+\theta}$. For any $x\in B_{n,i}, x'\in B_{n,i+1}$ it follows that 
\[\|x-x'\|\leq  \tfrac{r(n)}{2+\theta} + \tfrac{\theta r(n)}{2+\theta} + \tfrac{r(n)}{2+\theta}=r(n).\] 
This implies that $X_{j_{i}}\in \Xnear=\near(X_{j_{i+1}}, V_{j_{i}},r(n))$, which will cause the execution of the test $\collision(X_{j_i}, X_{j_{i+1}})$ (line~12 of \rrtstar). The latter will be evaluated to be true since $\B_\delta(\sigma_\eps)\subseteq \F$ and $r(n)\ll\delta$ (for $n$ large enough). Thus, in line~13 the edge $(X_{j_i},X_{j_{i+1}})$ will be added to the graph, unless there is a lower-cost alternative for connection. 
\end{proof}

Thus, $\EE^2$ guarantees that the \rrtstar tree $G$ contains a path  $\sigma'_n$ connecting $s$ to $t$ that follows $\sigma_\eps$ closely. In order to ensure that $c(\sigma'_n)\leq (1+\eps)c^*$ we need one more step, since $\sigma'_n$ could stay close to $\sigma_\eps$ but zig-zag around it, resulting in a high-cost solution. 

Define the constants $\alpha\in (0,\theta\eps /16),\beta\in (0,\theta\eps /16)$. Additionally, define for every $1\leq i\leq M_n$ the ball $B_{n,i}^\beta:=\B_{\frac{\beta r(n)}{2+\theta}}(x_i)$. The event $\EE^3$ (Definition~\ref{def:ee3}) indicates that a fraction of at most $\alpha$ of the smaller balls $B^\beta_{n,i}$ does not contain samples from $V_i$. We show that $\EE^3$ occurs with probability approaching $1$ (Lemma~\ref{lem:ee3}). We then proceed to show that if $\EE^2,\EE^3$ occur simultaneously then \rrtstar is guaranteed to return a solution with cost at most $(1+\eps)c^*$ (Lemma~\ref{lem:final}).

\subsection{Proof of Theorem~\ref{thm:main}}
We start with a formal definition of $\EE^1$:
\begin{definition}\label{def:ee1}
  For every $1\leq j \leq n$ denote by $\xrand^j,\xnew^j$ the random
  and new samples of \rrtstar in iteration $j$ (line~3 and line~5 in
  Algorithm~\ref{alg:rrt_star}, respectively). Define
  $n':=\mu n$ and 
  \begin{align*}
  \EE^1_n:=\{& \forall 1\leq i\leq M_n ,n'\leq j\leq n:\\ & \quad \textup{ if }\xrand^j\in B_{n,i} \textup{ then } \xrand^j = \xnew^j\}.\end{align*}

\noindent That is, $\EE^1_n$ is the event that all $\xrand^j\in B_{n,i}$ for $j$ between $n'$ and $n$ satisfy $\xrand^j=\xnew^j$. 
\end{definition} 

\begin{remark} 
We wish to stress that the following lemma, which lower bounds the probability of $\EE^1_n$, is a key ingredient in our proof. As we shall see below, this would allow us to treat some of the vertices added by \rrtstar as uniformly sampled, which is not true for all samples, as some are perturbed by the $\steer$ operation. We mention that this issue was not addressed in the original proof in~\cite{KF11}, where the \rrtstar nodes were assumed (incorrectly) to be uniformly distributed. Furthermore, setting the steering step $\eta=\infty$ does not resolve this issue.
\end{remark}

\begin{lemma}\label{lem:ee1}
  There exist two constants $a,b>0$ such that $\Pr[\EE^1_n]\geq 1-a\cdot e^{-bn}$.
\end{lemma}
\begin{proof}
  A similar proof appears in~\cite[Claim~6]{SK18}, albeit for a different type of sampling scheme and in the context of an \rrg analysis. The main challenge here is to show that while it is not true that all the new  samples $\xnew$ are distributed uniformly randomly (due to lines 4,5 in Algorithm~\ref{alg:rrt_star}), most of them are.  Define $\kappa:=\min\{\eta,\delta\}/10$ and set $z_1,\ldots, z_\ell$ to be a sequence of points placed along $\sigma_\eps$, such that $\ell=c(\sigma_\eps)/\kappa$, and $\|z_k-z_{k+1}\|\leq \kappa$. Observe that for $n$ large enough it holds that $\bigcup_{i=1}^{M_n}B_{n,i}\subset \bigcup_{k=1}^\ell\B_{\kappa}(z_k)$.

Denote by $V_{n'}^{\text{RRT}}$ the vertex set of \rrt after $n'$ iterations. 
Theorem~1 in~\cite{KSKBH18} states that there exist constants $a,c>0$ such that the probability that for every $1\leq k\leq \ell$ it holds that $V_{n'}^{\text{RRT}}\cap \B_\kappa(z_k)\neq \emptyset$ is at least 
$a\cdot e^{-cn'}=a\cdot e^{-bn}$, where $b:=c\mu$.
Notice that this theorem requires $\eta$ to be fixed (i.e., independent of $n$) and strictly positive.  

Denote the latter event to be ${\EE'}^1_n$. Next, we show that ${\EE'}^1_n$ implies $\EE^1_n$. First,  observe that $V_{n'}^{\text{RRT}}=V_{n'}^{\text{RRT}^*}$, where the latter is the vertex set of \rrtstar after $n'$ iterations, and assume that ${\EE'}^1_n$ holds . Fix an iteration $n'<j<n$ and some $1\leq k\leq \ell$. Due to the fact that $\eta>0$ is fixed, by the proof of Lemma~1 in~\cite{KSKBH18} it follows that if $\xrand^j\in \B_\kappa(z_j)$ then $\xnear^j\in \B_{5\kappa}(z_j)$, and consequently 
\begin{align*}\|\xrand^j-\xnear^j\| & = \|\xrand^j-z_j+z_j-\xnear^j\|\\ &\leq   \|\xrand^j-z_j\| +\| z_j-\xnear^j\| \leq \kappa + 5\kappa \leq \eta. \end{align*}
This implies that $\xnew^j=\xrand^j$. Additionally, observe that due to the fact that the straight-line path from $\xnear^j$ to $\xrand^j$ is contained in $\B_{\kappa}(z_j)$, where $\kappa < \delta /5$, it is also collision free. Thus, at the end of iteration $j$, $\xrand$ will be added to the \rrtstar graph as a vertex.
\end{proof}

We will prove that the following event $\EE^2$ holds with probability approaching $1$ by conditioning on $\EE^1$.
\begin{definition}\label{def:ee2}
  $\EE_n^2$ represents the event that every $B_{n,i}$ contains at least one vertex from $V_i$. That is,
  \[\EE_n^2:=\{\forall 1\leq i\leq M_n, V_i\cap B_{n,i}\neq \emptyset\}.\]
\end{definition}

\begin{lemma}\label{lem:ee2}
$\limn \Pr[\EE^2_n]=1$.
\end{lemma}
\begin{proof}
  Observe that 
  \begin{align*}\Pr[\EE^2_n]& =\Pr[\EE^2_n|\EE^1_n]\cdot \Pr[\EE^1_n]+\Pr[\EE^2_n|\overline{\EE^1_n}]\cdot\Pr[\overline{\EE^1_n}] \\ & \geq \Pr[\EE^2_n|\EE^1_n]\cdot \Pr[\EE^1_n].\end{align*}
  We shall lower-bound the expression $\Pr[\EE^2_n|\EE^1_n]$. By definition of $\EE^1_n$, for every $n'<j\leq n$, and $i$ such that $j\in T_i$, if $\xrand^j\in B_{n,i}$, then $\xnew^j=\xrand^j$ is a valid vertex of the \rrtstar graph. Thus, by conditioning on $\EE^1_n$ we can treat $V\setminus V_0$ as uniform random samples from $\F$. This will come in handy in bounding the probability of $\EE^2$:
    {\small\begin{align}
        \Pr&[\overline{\EE^2_n}|\EE^1_n] =\Pr\left[\exists 1\leq i\leq M_n, V_i\cap B_{n,i}=\emptyset\right] \nonumber \\ &\leq \sum_{i=1}^{M_n}\Pr[V_i\cap B_{n,i}=\emptyset] 
        = \sum_{i=1}^{M_n}\left(1-\frac{|B_{n,i}|}{|\F|}\right)^{|T_i|}  \label{eq:uniform} \\ 
&\leq  M_n\left(1-\tfrac{\zeta_d\left(\frac{r(n)}{2+\theta}\right)^d}{|\F|}\right)^{(n-n')/M_n} \nonumber \\ 
& \leq M_n\exp\left\{-\frac{n-n'}{M_n}\cdot \frac{\zeta_d}{|\F|}\cdot \frac{r(n)^d}{(2+\theta)^d} \right\} \label{eq:1-x} \\ 
& \leq  M_n\exp\left\{-\frac{n r(n)\theta (1-\mu)}{c(\sigma_\eps)(2+\theta)}\cdot \frac{\zeta_d}{|\F|}\cdot \frac{r(n)^d}{(2+\theta)^d}\right\}  \nonumber \\ & = M_n\exp\left\{-\frac{\theta\zeta_d(1-\mu)}{c(\sigma_\eps)(2+\theta)^{d+1}|\F|}\cdot n \cdot r(n)^{d+1}\right\} \nonumber \\ 
& =: M_n\exp\left\{-\xi\cdot n \cdot \gamma^{d+1}\frac{\log n}{n}\right\} \label{eq:xi} \\ & = \left\lceil c(\sigma_\eps)\cdot\left(\tfrac{r(n)}{2+\theta}\right)^{-1}\right\rceil\exp\left\{-\xi \gamma^{d+1}\log n\right\} \nonumber  \\ & < \left( c(\sigma_\eps)\cdot\left(\tfrac{r(n)}{2+\theta}\right)^{-1}+1\right)\exp\left\{-\xi \gamma^{d+1}\log n\right\} \nonumber  \\ 
&=\tfrac{c(\sigma_\eps)(2+\theta)}{\theta\gamma}(\log n)^{-1/(d+1)} n^{1/(d+1)-\xi \gamma^{d+1}} \nonumber \\ &\quad \quad+\exp\left\{-\xi \gamma^{d+1}\log n\right\} , \label{eq:E2-val}\end{align}}where (\ref{eq:uniform}) is due to the union bound 
and the fact that $V_i$ is uniformly sampled at random from $\F$, (\ref{eq:1-x}) is due to the inequality $1-x\leq e^{-x}$ for $x\in (0,1)$ which applies here for $n$ large enough, and (\ref{eq:xi}) defines $\xi:=\frac{\theta\zeta_d(1-\mu)}{c(\sigma_\eps)(2+\theta)^{d+1}|\F|}$. If $(d+1)^{-1}-\xi \gamma^{d+1} \leq  0$ then the final expression tends to 0. Indeed, 
    {\small \begin{align*}
     \frac{1}{d+1}-\frac{\theta\zeta_d(1-\mu)}{c(\sigma_\eps)(2+\theta)^{d+1}|\F|}\cdot \gamma^{d+1}\leq 0 \iff \\
         (2+\theta)\left(\tfrac{c(\sigma_\eps)|\F|}{(d+1)\theta\zeta_d(1-\mu)}\right)^{\frac{1}{d+1}}\leq  (2+\theta)\left(\tfrac{c^*(1+\eps/4)|\F|}{(d+1)\theta\zeta_d(1-\mu)}\right)^{\frac{1}{d+1}} \leq  \gamma.
    \end{align*}}

It remains to show that $\limn \Pr[\EE^2_n|\EE^1_n]\cdot \Pr[\EE^1_n] = 1$:
\begin{align*}
  \Pr[\EE^2_n|\EE^1_n]&\cdot \Pr[\EE^1_n]  = (1-\Pr[\overline{\EE^2_n}|\EE^1_n])(1-\Pr[\overline{\EE^1_n}]) \\
& = 1 + \Pr[\overline{\EE^2_n}|\EE^1_n]\cdot \Pr[\overline{\EE^1_n}] -\Pr[\overline{\EE^2_n}|\EE^1_n] - \Pr[\overline{\EE^1_n}] \\ & > 1 -\Pr[\overline{\EE^2_n}|\EE^1_n] - \Pr[\overline{\EE^1_n}], 
\end{align*}
where the final expression converges to $1$, according to Equation~\ref{eq:E2-val} and Lemma~\ref{lem:ee1}.
\end{proof}

Next we consider the existence of samples in a collection of smaller balls. 

\begin{definition}\label{def:ee3}
  Let {\small $K_n^\beta:=|\{i\in \{1,\ldots,M_n\}: B_{n,i}^\beta \cap
  V_i=\emptyset\}|$}. $\EE_n^3:=\{K_n^\beta \leq  \alpha M_n\}$ is the event that at most $\alpha M_n$ of the smaller balls $B_{n,i}^\beta$ do not contain any samples from $V_i$.
\end{definition}

\begin{lemma}\label{lem:ee3}
    $\limn \Pr[\EE^3_n]=1$.
\end{lemma}
\begin{proof}
  Similarly to Lemma~\ref{lem:ee2}, it is sufficient to show that $\limn \Pr[\overline{\EE^3_n}|\EE^1_n]=0$. We shall upper bound the probability that $K_n^\beta > \alpha M_n$ assuming that $\EE^1_n$ holds. To this end, we compute the expectation of $K_n^\beta$ and apply Markov's inequality.

  For every $1\leq i\leq M_n$, denote by $I_i$ the indicator variable for the event that $B_{n,i}^\beta \cap V_i=\emptyset$. Observe that $K_n^\beta=\sum_{i=1}^{M_n}I_i$. For $n$ large enough we have that 
  {\small \begin{align*}
  E[I_i]& = \Pr[I_i=1] =\left(1-\tfrac{|B_{n,i}^\beta|}{|\F|}\right)^{|T_i|} \\ 
& \leq \left(1-\tfrac{\beta^d\zeta_d\left(\frac{r(n)}{2+\theta}\right)^d}{|\F|}\right)^{n(1-\mu)/M_n} \\
& \leq \exp\left\{-\tfrac{\beta^d\theta\zeta_d(1-\mu)}{c(\sigma_\eps)(2+\theta)^{d+1}|\F|}\cdot n \cdot r(n)^{d+1}\right\} \\
& \leq \exp\left\{-\tfrac{\beta^d\theta\zeta_d(1-\mu)}{c(\sigma_\eps)(2+\theta)^{d+1}|\F|}\cdot \gamma^{d+1} \cdot \log n \right\} \\ 
& = \exp\left\{-\tfrac{\beta^d}{d+1} \log n \right\}  = n^{-\beta^d / (d+1)}. 
  \end{align*}}Thus, $E[K_n^\beta]=\sum_{i=1}^{M_n}E[I_i]\leq M_n n^{-\beta^d / (d+1)}$.
  By Markov's inequality, it follows that 
   {\small \begin{align}\Pr[K_n^\beta >  \alpha M_n] \leq \tfrac{E[K_n^\beta]}{\alpha M_n}\leq \tfrac{M_nn^{-\beta^d / (d+1)}}{\alpha M_n}=\tfrac{n^{-\beta^d / (d+1)}}{\alpha}.\label{eq:markov}\end{align}}As $\alpha$ is fixed, the last expression tends to $0$ as $n$ tends to $\infty$.
While the upper bound obtained in~\eqref{eq:markov} is sufficient for our purpose, we mention that a tighter bound can be derived by using a slightly more complex Poissonization argument similar to that used in~\cite{JSCP15}.
\end{proof}

Next, we show that if $\EE^2,\EE^3$ occur simultaneously, then the cost of $c(\sigma_n)$ is bounded by $(1+\eps)c^*$.

\begin{lemma}\label{lem:final}
  For $n$ large enough, if the events $\EE^2_n,\EE^3_n$ occur, then $c(\sigma_n)\leq (1+\eps)c^*$.
\end{lemma}
\begin{proof}
  As $\EE^2_n\wedge \EE^3_n$ we may define the sequence of vertices $X_{j_1},\ldots, X_{j_{M_n}}\in V$, such that $X_{j_1}=s,X_{j_{M_n}}=t$, and for every $1< i<M_n$, $X_{j_i}\in V_i\cap B_{n,i}^\beta$  if $ V_i\cap B_{n,i}^\beta\neq \emptyset$, and $X_{j_i}\in V_i\cap B_{n,i}$ otherwise.  

Denote by $\sigma'_n$ the path induced by concatenating those points, and notice that it is collision free by definition of $B_{n,i}$ and $\sigma_\eps$. Next, we claim that the cost of the path $\sigma_n$ obtained by \rrtstar is upper-bounded by the cost of $\sigma'_n$, which is equal to $\sum_{i=2}^{M_n}\|X_{j_i}-X_{j_{i-1}}\|$. Consider iteration $j_i$ of \rrtstar, for $1< i\leq M_n$ and observe that (i) $\xnew^{j_i}=X_{j_i}$, (ii) $X_{j_{i-1}}\in \Xnear^{j_i}$. By Claim~\ref{clm:connection}, it follows that $\cost(X_{j_i})\leq \sum_{k=2}^{i}\|X_{j(k)}-X_{j(k-1)}\|$, as desired. Thus, $c(\sigma_n)\leq c(\sigma_n')$.

We proceed to bound $c(\sigma_n')$. Observe that for any $1< i \leq M_n$ it holds that $\|X_{j_i}-X_{j_{i-1}}\|$ is at most
{\small \[ 
\begin{cases}
  \frac{\theta r(n)}{2+\theta}+ \frac{\beta r(n)}{2+\theta} + \frac{\beta r(n)}{2+\theta}, &  \text{if } X_{j_{i-1}}\in B_{n,i-1}^\beta \text{ AND }  X_{j_i}\in B_{n,i}^\beta\\
  \frac{\theta r(n)}{2+\theta}+ \frac{\beta r(n)}{2+\theta} + \frac{r(n)}{2+\theta}, &  \text{if } X_{j_{i-1}}\in B_{n,i-1}^\beta \text{ XOR } X_{j_i}\in B_{n,i}^\beta\\
  \frac{\theta r(n)}{2+\theta}+ \frac{r(n)}{2+\theta} + \frac{r(n)}{2+\theta}, &  \text{otherwise}\\
  \end{cases}
.\]}
Thus,
{\small \begin{align*}
  &c(\sigma'_n) \leq \sum_{i=2}^{M_n}\|X_{j_i}-X_{j_{i-1}}\| \\ 
  & \leq (M_{n}-1)\tfrac{\theta r(n)}{2+\theta}+\lceil(1-\alpha)(M_n-1)\rceil \tfrac{2\beta r(n)}{2+\theta}\\ & \quad \quad + \lfloor\alpha(M_n-1)\rfloor \tfrac{2 r(n)}{2+\theta} 
 \leq (M_n-1)r(n)\frac{\theta +2\beta +2\alpha }{2+\theta} \\ & \leq \frac{c(\sigma_\eps)(2+\theta)}{\theta r(n)} r(n)  \frac{\theta +2\beta +2\alpha }{2+\theta}  \leq \left(1+\tfrac{\eps}{4}\right)c^* \cdot  \frac{\theta +2\beta +2\alpha}{\theta}  \\ &< \left(1+\tfrac{\eps}{4}\right)c^* \frac{\theta +\tfrac{2\theta \eps}{16} +\frac{2\theta \eps}{16}}{\theta} = \left(1+\tfrac{\eps}{4}\right)^2c^* \\ & =\left(1+\tfrac{\eps}{2}+\tfrac{\eps^2}{16}\right) c^*  
 < \left(1+\tfrac{\eps}{2}+\tfrac{\eps}{16}\right) c^* < (1+\eps)c^*.\qedhere
\end{align*}} 
\end{proof}

It remains to show that $\EE^2\wedge \EE^3$ occurs with probability approaching $1$:
\begin{align*}
  \limn\Pr[\EE^2\wedge \EE^3] & =1- \limn \Pr[\overline{\EE^2}\vee \overline{\EE^3}] 
   \\ &\geq 1- \limn \left(\Pr[\overline{\EE^2}] + \Pr[\overline{\EE^3}]\right) = 1.
\end{align*}

\section{Conclusion}\label{sec:conclusion}
In this paper we revisited the original asymptotic-optimality proof of \rrtstar in~\cite{KF11}, and discussed an apparent logical gap within it. We then introduced an alternative proof that amends this logical gap. Our new proof suggests that the connection radius of \rrtstar should be slightly larger 
than the original bound on the radius that was developed in~\cite{KF11}. We leave the question of whether our bound is tight, i.e., whether the exponent of $1/(d+1)$ in Equation~\eqref{eq:r_n} can be lowered to $1/d$, to future research. The practical successes of the algorithm and its extensions, using the exponent $1/d$, provide some evidence that this might be the case.

\section*{Acknowledgments}
We thank Sertac Karaman for insightful discussions on his work~\cite{KF11}. We also thank Michal Kleinbort for feedback on the manuscript. This work was supported in part by NSF, Award Number: 1931815.

\appendix
We provide a detailed counter example (Figures~\ref{fig:ex1}-\ref{fig:ex10}) illustrating our argument that the fact that for every $1\leq i < M_n$ (i) there exist $X_{j_i},X_{j_{i+1}}$ such that $X_{j_i}\in B_{n,i},X_{j_{i+1}}\in B_{n,i+1}$ and (ii)  $j_i<j_{i+1}$, does not necessarily mean that (iii) there exists a sequence $j_1\leq j_2\leq \ldots \leq j_{M_n}$ such that $X_{j_i}\in B_{n,i}$ for every $1\leq i < M_n$ (see Section~\ref{sec:issue}). 

\begin{figure*}
  \centering
    \includegraphics[width=1.2\columnwidth, page=1]{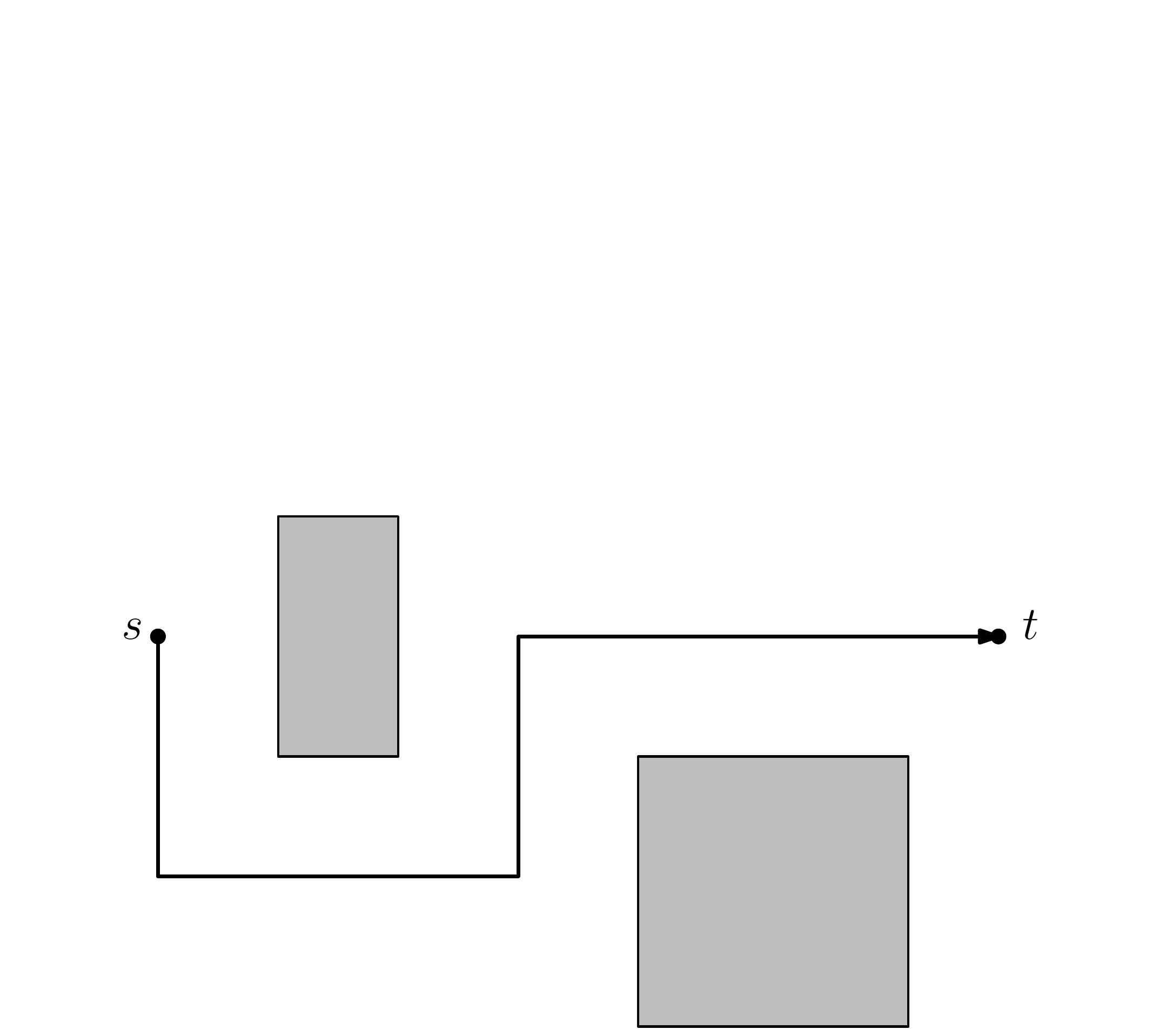}
  \caption{The input scenario for the counter example. The goal is to find a path from configuration $s$ on the left to $t$ on the right, while avoiding the two gray obstacles. The path $\sigma_{\eps}$ is drawn as a black curve.}
  \label{fig:ex1}
\end{figure*}
\begin{figure*}
  \centering
    \includegraphics[width=1.2\columnwidth, page=2]{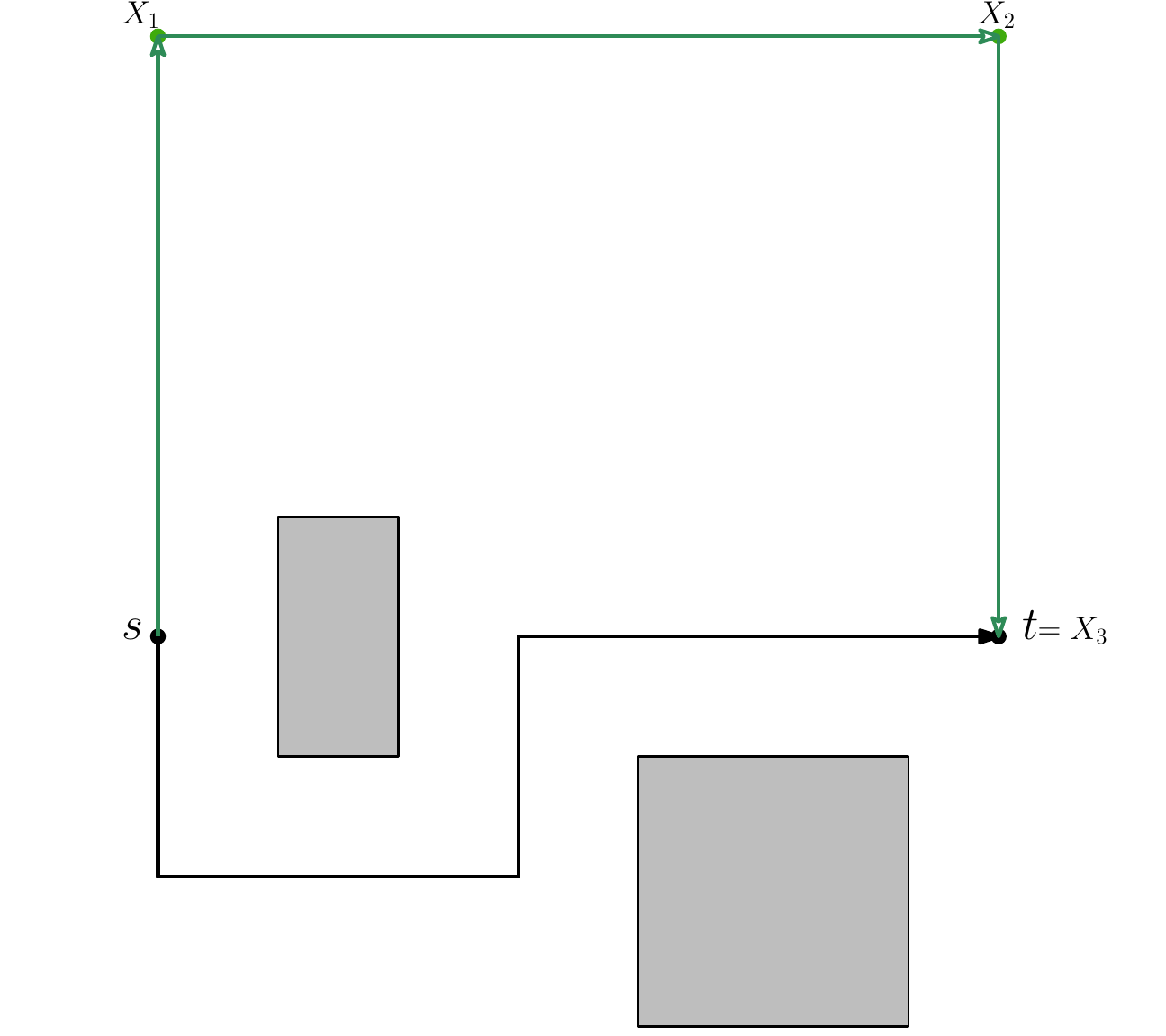}
  \caption{The first three samples $X_1,X_2,X_3$ are drawn by the algorithm, where $X_3=t$. The edge $(s,X_1)$ is added first through line~5 of Algorithm~\ref{alg:rrt_star}. The edges $(X_1,X_2),(X_2,X_3)$ are added in a similar fashion. We  assume that no rewiring occurs in those steps due to the smaller magnitude of $r_1$ in comparison to $\|X_1-X_3\|$. We also mention that the length of the new path to $t$ just discovered can be made arbitrarily long with respect to $\sigma_\eps$ by moving $X_1,X_2$ further away from $s$ and $t$ respectively, and setting the steering parameter $\eta$ to be large enough to support such long connections. }
  \label{fig:ex2}
\end{figure*}
\begin{figure*}
  \centering
    \includegraphics[width=1.2\columnwidth, page=3]{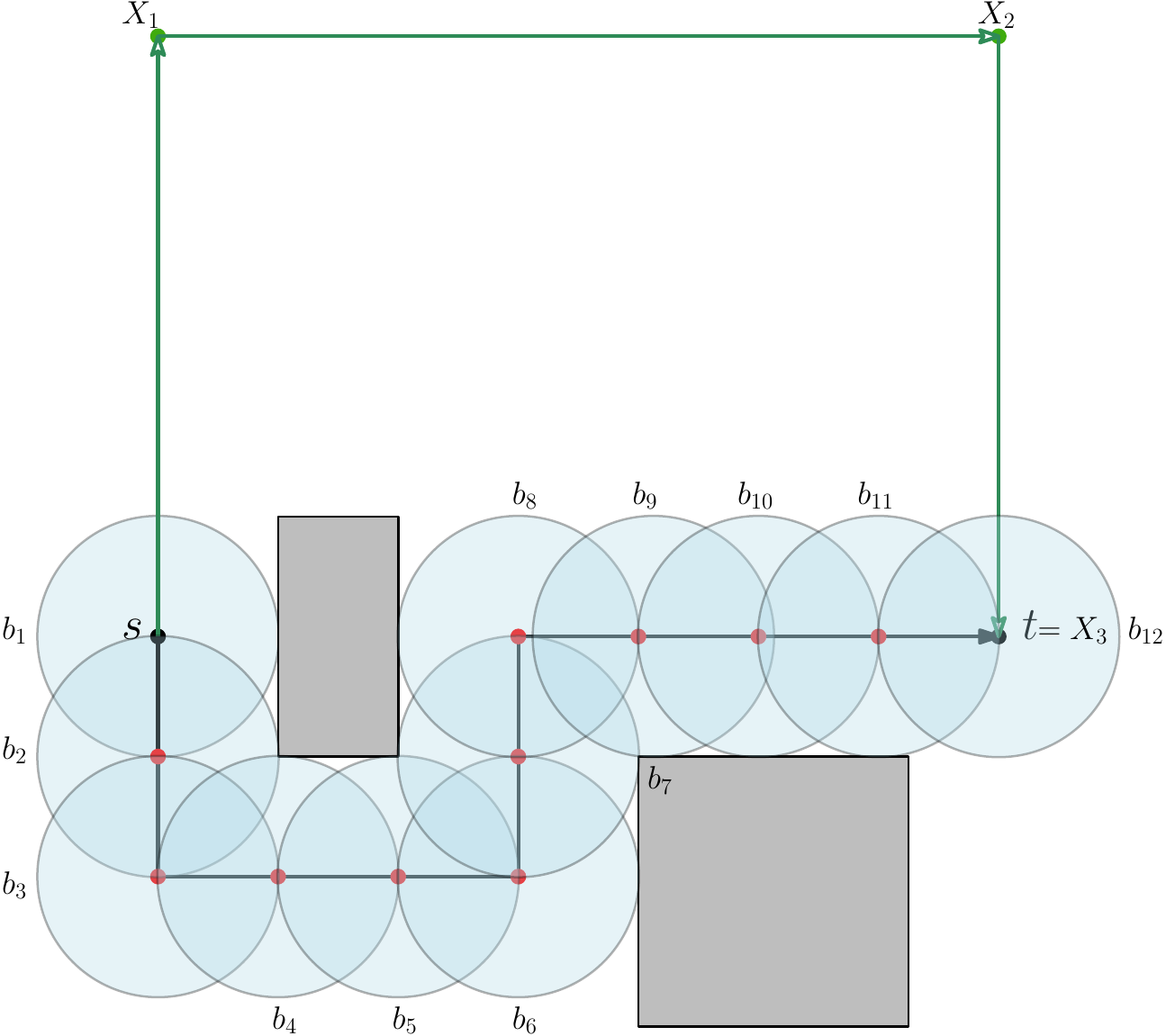}
  \caption{We construct a sequence of $M_{n}$ balls $B_{n,1},\ldots, B_{n,M_n}$, which we denote for simplicity by $b_1,\ldots, b_{M_n}$, and we fix $n=23$. For simplicity, we set $M_n=12$ in the illustration, to avoid unnecessarily complicating the visualization. Below we also assume that the connection radius $r_{23}$ used by \rrtstar is equal to the diameter of any ball $b_i$, although a similar out come will follow when $r_{23}$ is much larger (as long as $r_{23}<\|X_2-X_3\|$).}
  \label{fig:ex3}
\end{figure*}
\begin{figure*}
  \centering
    \includegraphics[width=1.2\columnwidth, page=4]{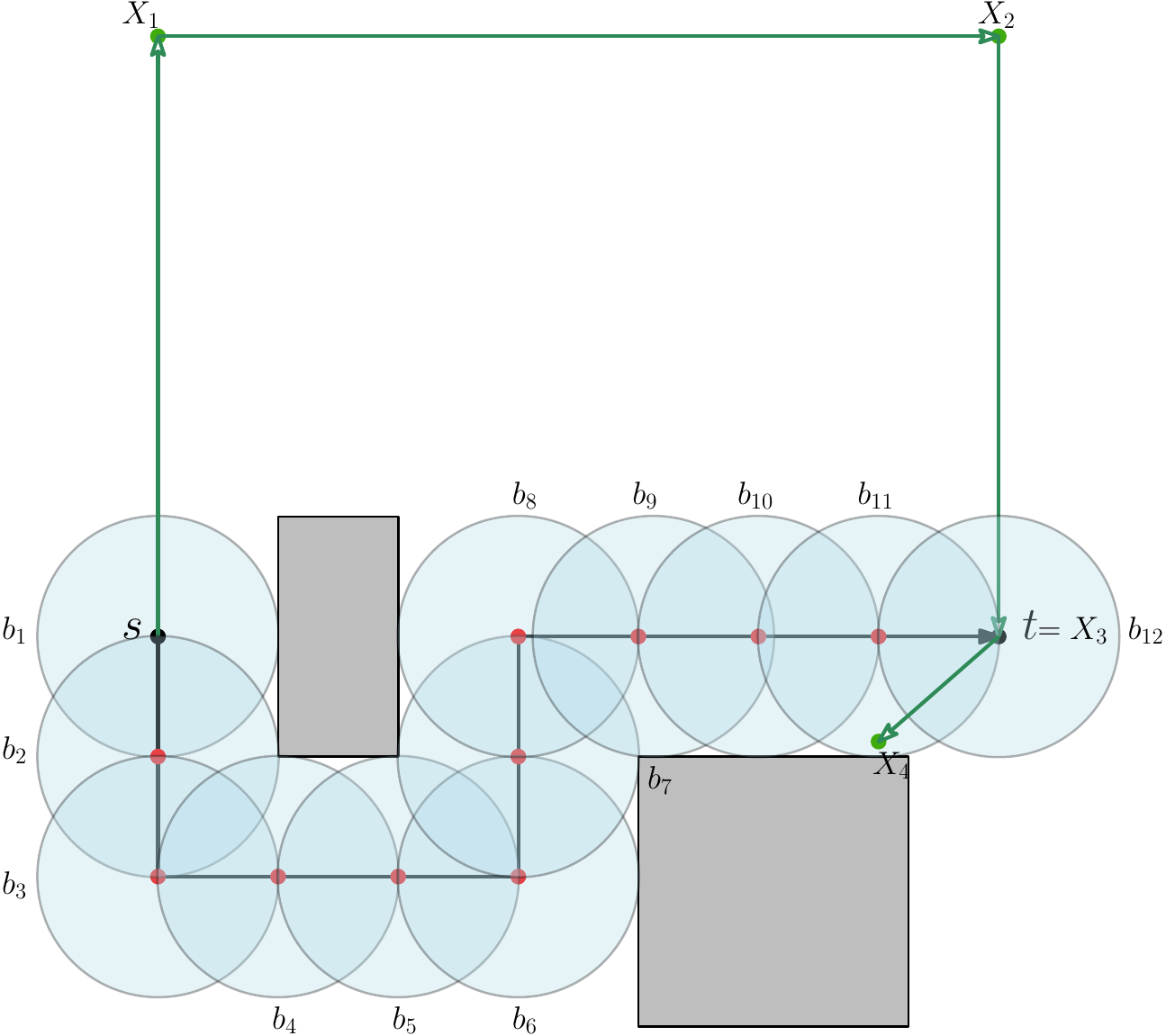}
  \caption{Next, we generate the sample $X_4\in b_{11}$, which introduces the edge $(X_3,X_4)$ and does not result in rewiring. 
As we mentioned earlier, we assume that $r_{n}<\|X_2-X_3\|$, which implies that the edge $(X_2,X_4)$ will not be considered.}
  \label{fig:ex4}
\end{figure*}
\begin{figure*}
  \centering
    \includegraphics[width=1.2\columnwidth, page=5]{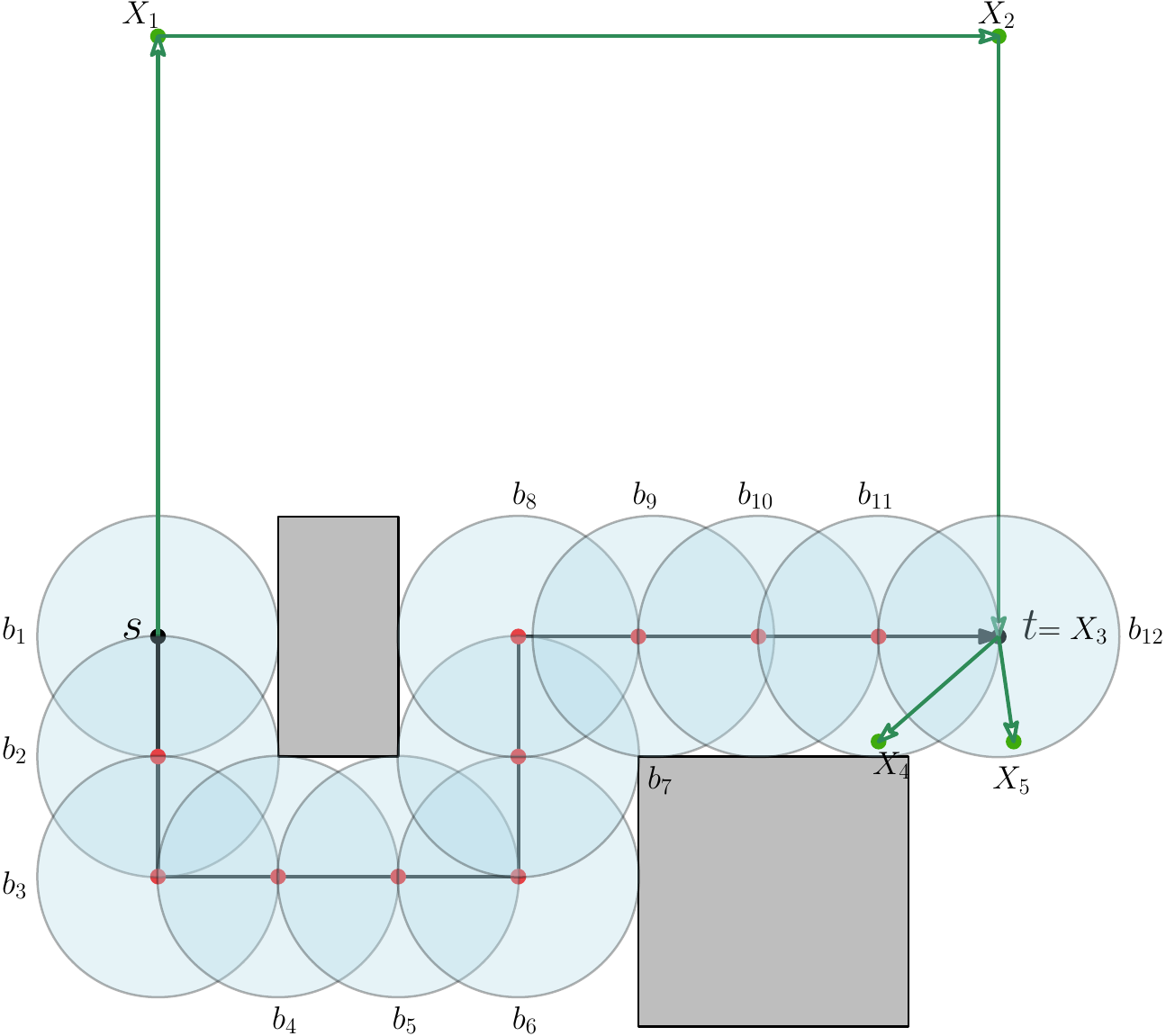}
  \caption{Next, we generate the sample $X_5\in b_{12}$, which introduces the edge $(X_4,X_5)$, since the cost-to-come via $X_3$ is smaller than through a connection from $X_4$. Clearly, the edge $(X_5,X_4)$ cannot improve the cost-to-come of $X_4$, and it is therefore not added in the rewiring stage. Observe that $X_4\in b_{11},X_5\in b_{12}$, and $X_4$ was sampled before $X_5$, which implies that conditions (i), (ii) are satisfied locally for $b_{11},b_{12}$.}
  \label{fig:ex5}
\end{figure*}

\begin{figure*}
  \centering
    \includegraphics[width=1.2\columnwidth, page=6]{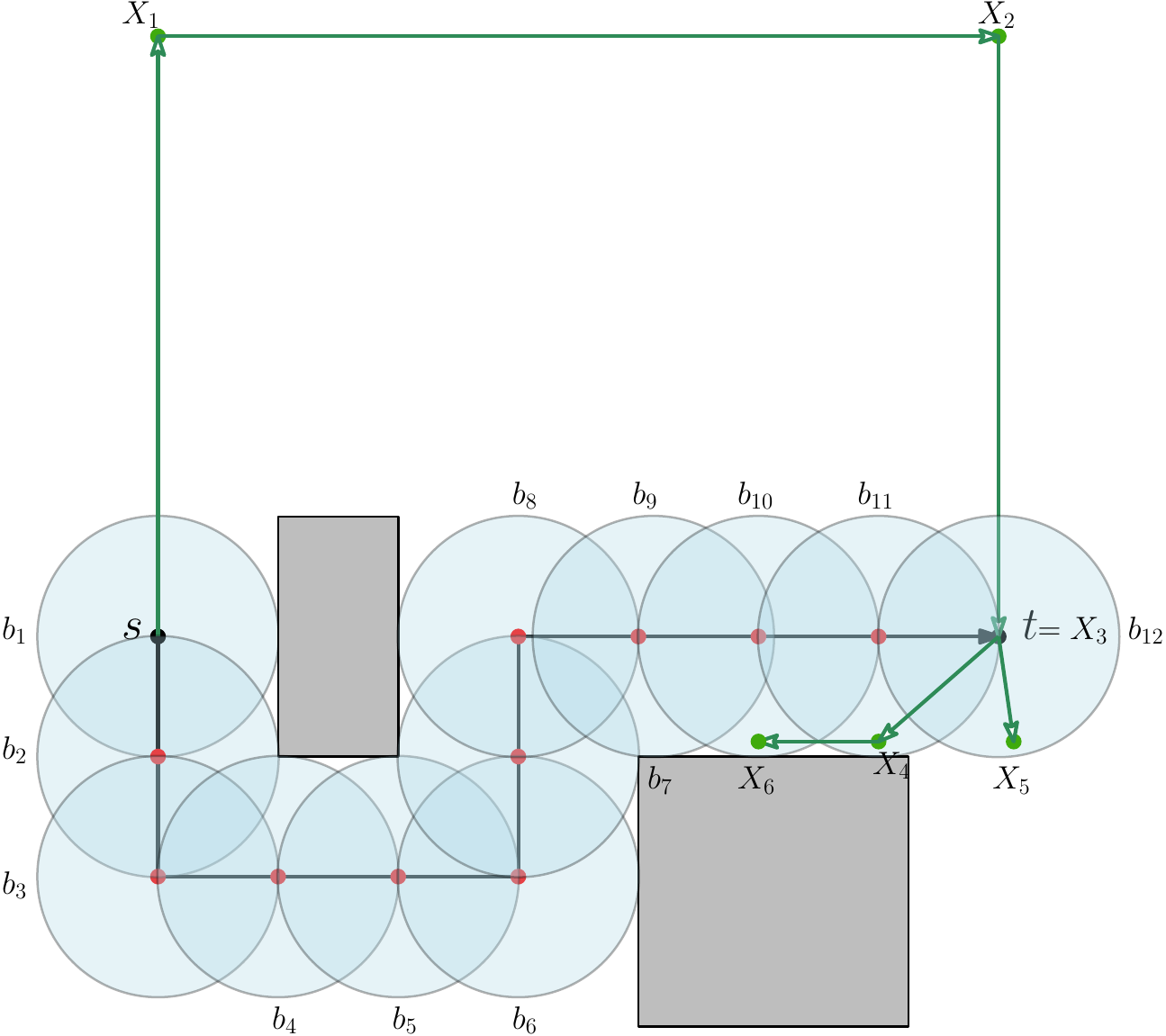}
  \caption{Similarly to $X_4$, the sample $X_6$ is produced in $b_{10}$, which yields the edge $(X_4,X_6)$, and introduces no rewiring. }
  \label{fig:ex6}
\end{figure*}

\begin{figure*}
  \centering
    \includegraphics[width=1.2\columnwidth, page=7]{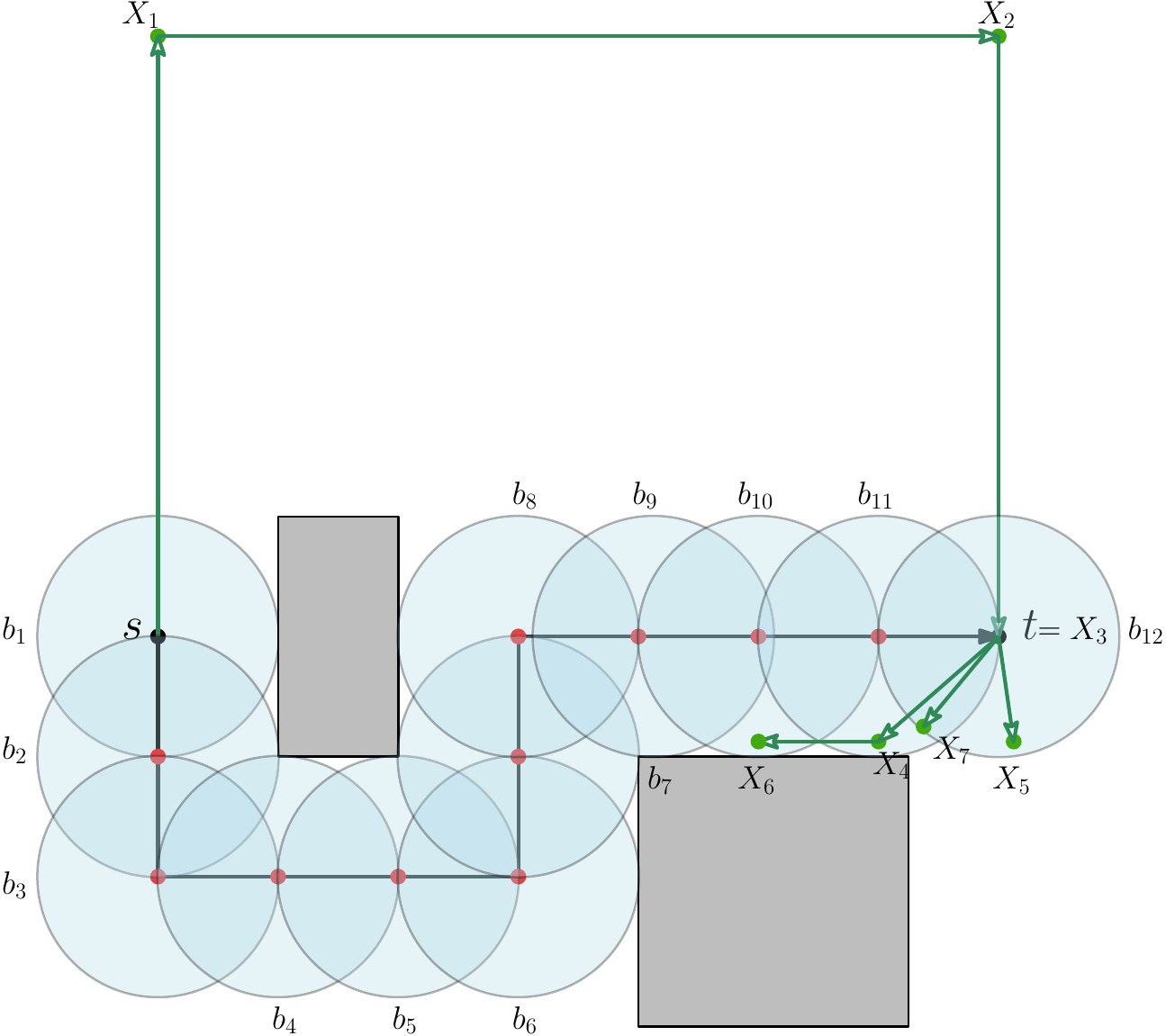}
  \caption{Similarly to $X_5$, the sample $X_7$ is produced in $b_{11}$, and the edge $(X_4,X_6)$ is left intact.}
  \label{fig:ex7}
\end{figure*}

\begin{figure*}
  \centering
    \includegraphics[width=1.2\columnwidth, page=8]{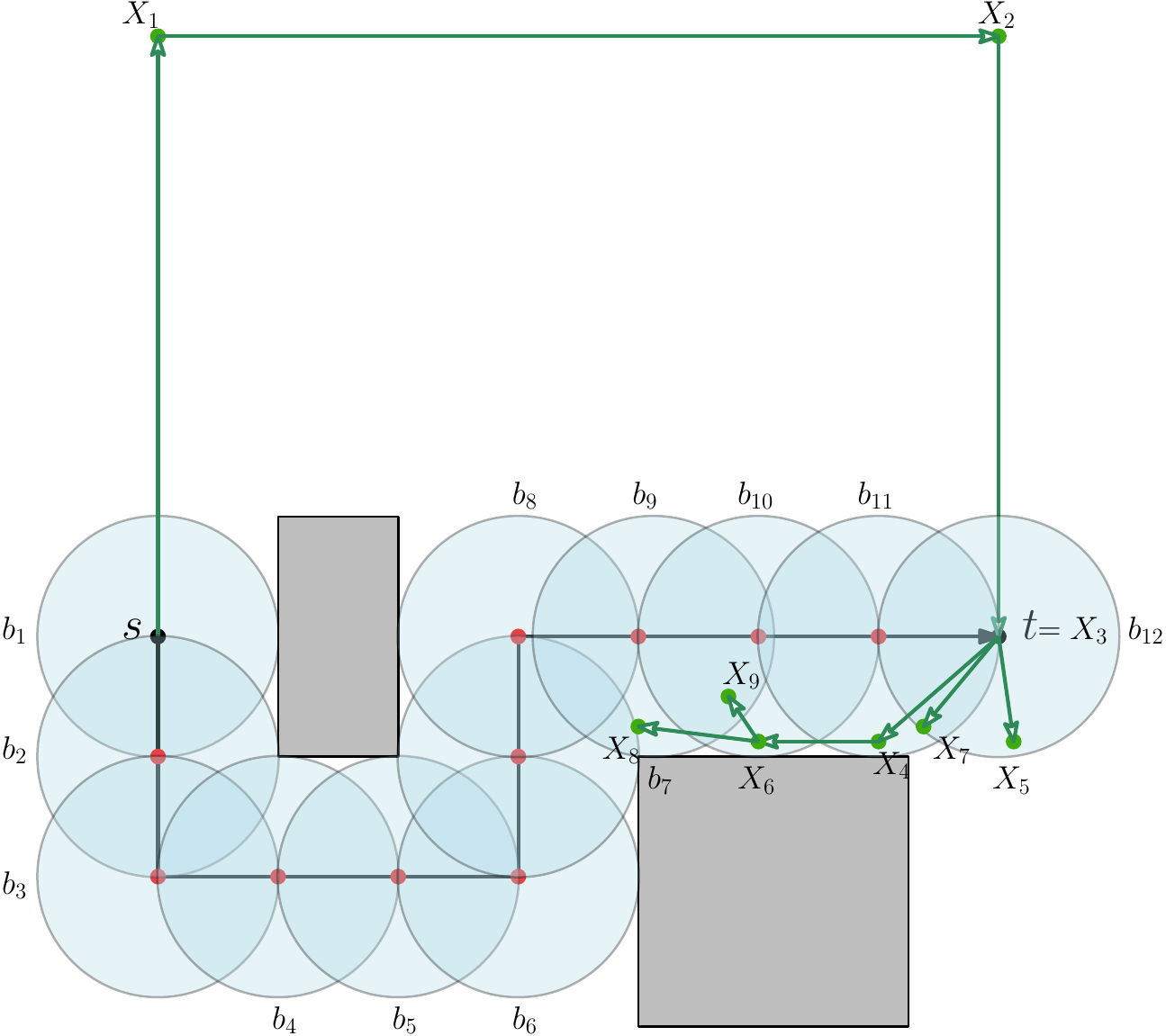}
  \caption{In a similar manner, we introduce incrementally the samples $X_8\in b_9, X_9\in b_{10}$. Notice that a path from $t$ in the opposite direction of the balls towards $s$  (currently till $X_8$) is beginning to form. }
  \label{fig:ex8}
\end{figure*}

\begin{figure*}
  \centering
    \includegraphics[width=1.2\columnwidth, page=9]{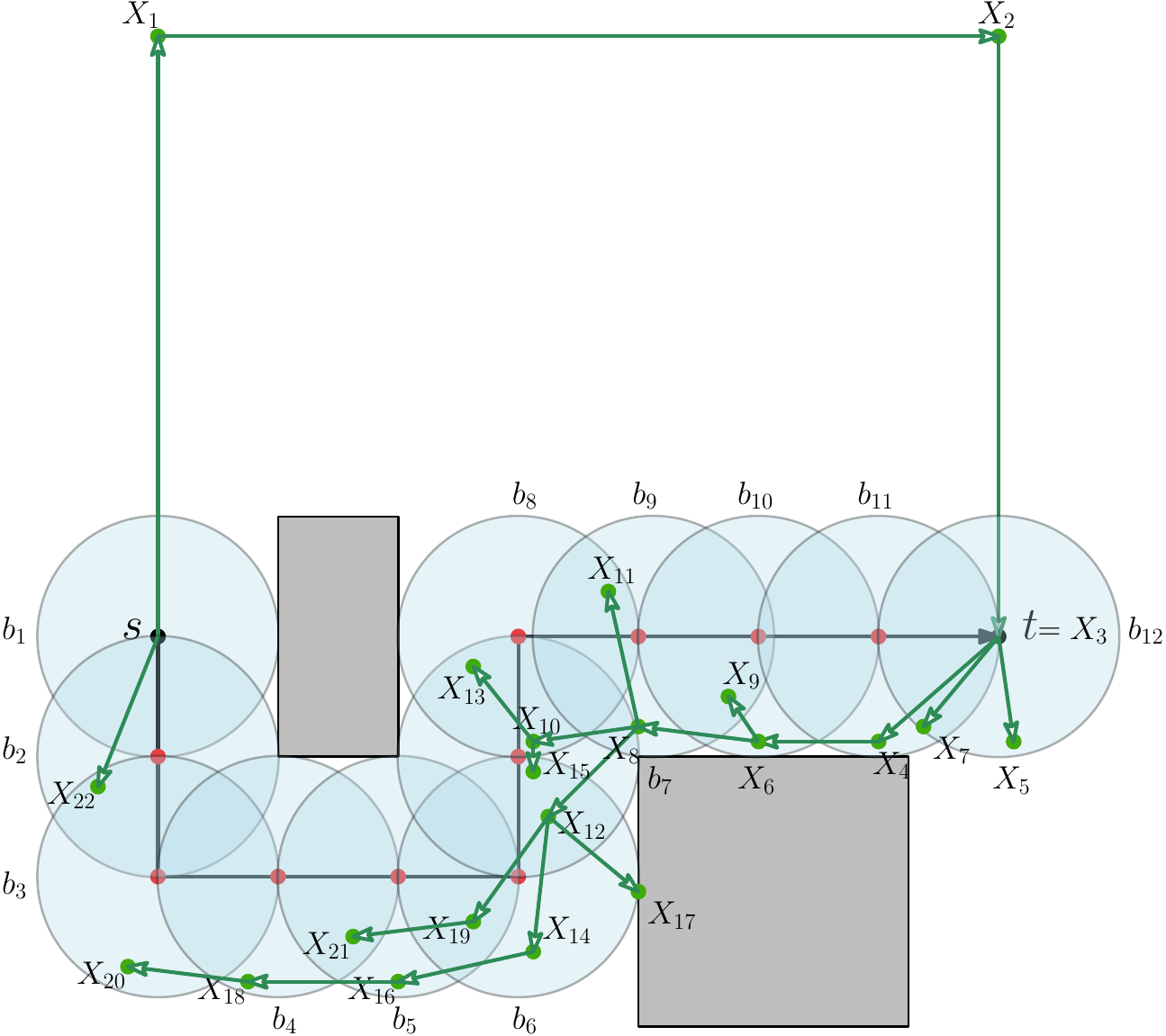}
  \caption{This sample scheme can be repeated until the sample $X_{22}\in b_2$ is produced, which is within $r_{23}$ distance from $s$. This introduces the edge $(s,X_{22})$, which minimizes the cost-to-come to $X_{22}$.}
  \label{fig:ex9}
\end{figure*}

\begin{figure*}
  \centering
    \includegraphics[width=1.2\columnwidth, page=10]{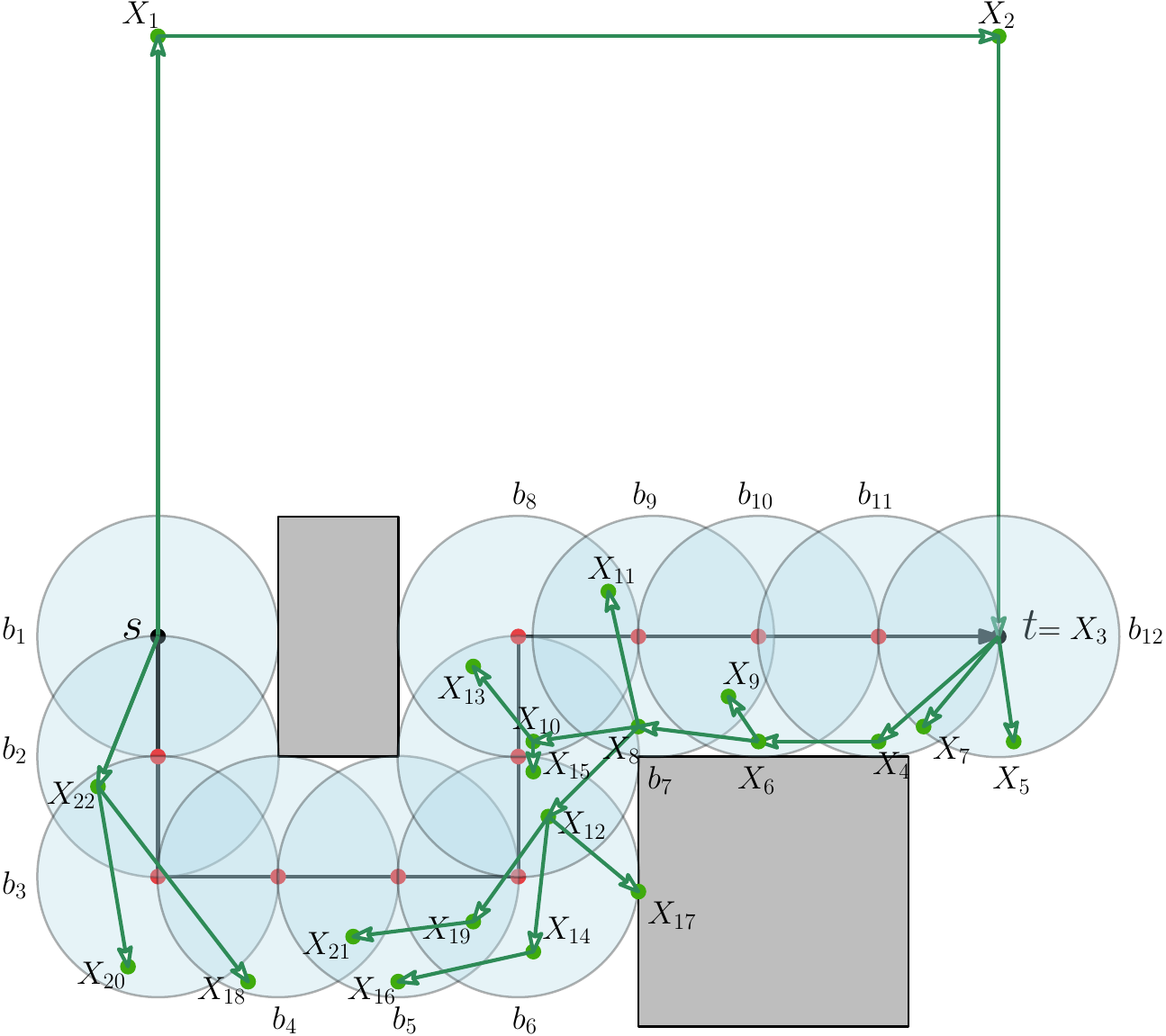}
  \caption{The introduction of $X_{22}$ forces a rewire of $X_{20}$ and $X_{18}$ within the same iteration: the edges $(X_{22},X_{20}),(X_{22},X_{18})$ are added, whereas $(X_{18},X_{20})$ and $(X_{16},X_{18})$ are removed. It is important to note that by definition of \rrtstar, this rewire does not promote further rewiring to the predecessors of $X_{18}, X_{20}$ in $G$. }
  \label{fig:ex9}
\end{figure*}

\begin{figure*}
  \centering \includegraphics[width=1.2\columnwidth, page=11]{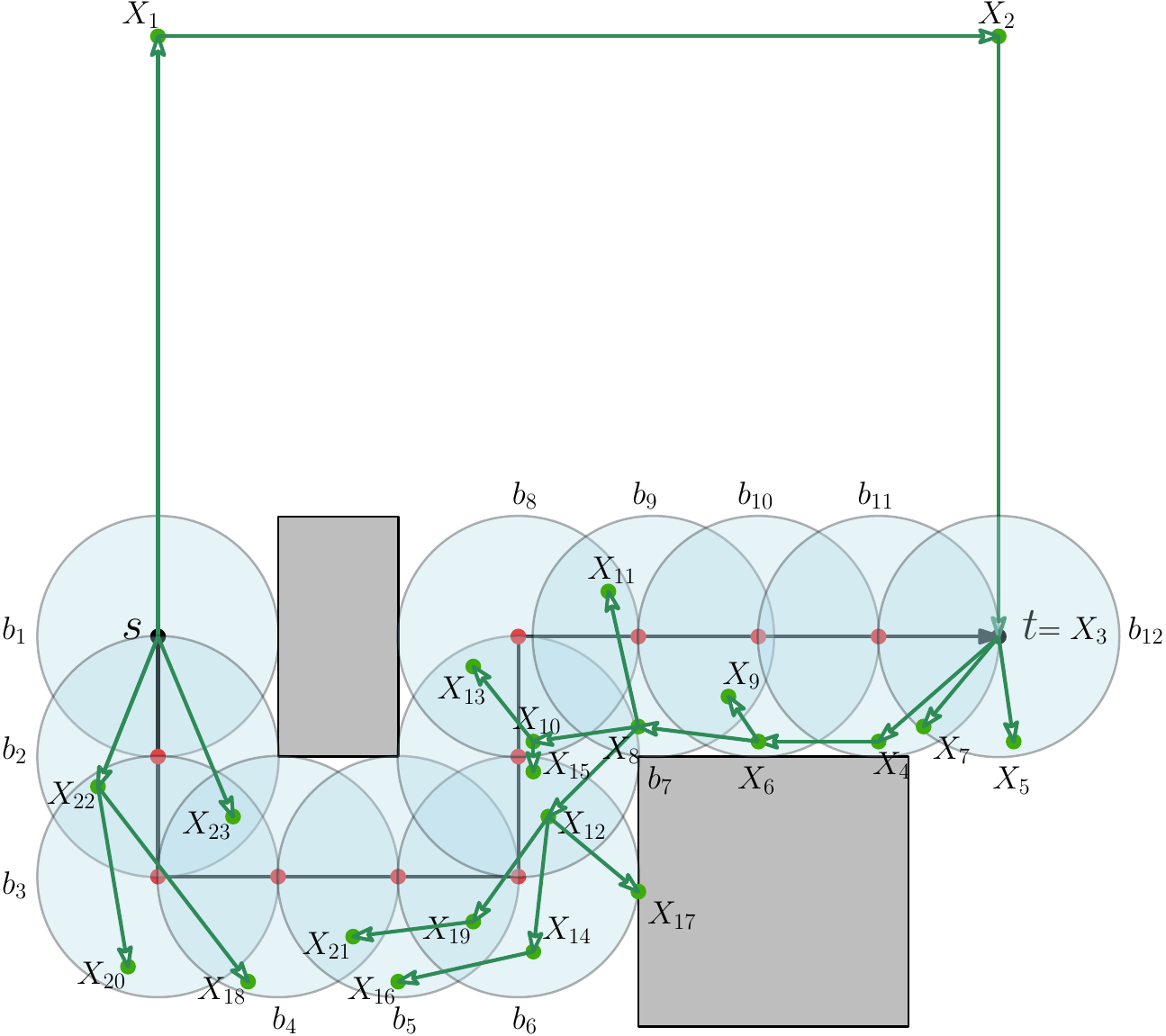}
  \caption{Finally, the sample $X_{23}\in b_3$ is drawn, and the edge $(s,X_{23})$ is added. This will promote additional rewires to $X_{16}, X_{19}, X_{18},X_{21}$,in the vicinity of $X_{23}$, although as earlier those rewires will not propagate to other vertices of $G$. 
It is clear that at this point for every $1\leq i<M_n$ it holds that there exist (i) $X_{j_i}\in b_i,X_{j_{i+1}}\in b_{i+1}$ and (ii) $j_i<j_{i+1}$. Unfortunately, the graph $G$ does not contain a path starting at $s$ and going sequentially through the balls $b_1,\ldots, b_{12}$ until $t$ is reached. To conclude, even though conditions (i), (ii) hold, \rrtstar will return the path consisting of the vertices $s,X_1,X_2,t$, which is substantially longer than $\sigma_\eps$.}
  \label{fig:ex10}
\end{figure*}

\bibliographystyle{IEEEtran}
\bibliography{bibliography}

\end{document}